\documentclass{article}

\usepackage{microtype}
\usepackage{graphicx}
\usepackage{subcaption}
\usepackage{booktabs} 

\usepackage{multirow}
\usepackage[table]{xcolor} 
\usepackage{siunitx}    

\usepackage{hyperref}

\usepackage[preprint]{icml2026}

\usepackage{amsmath}
\usepackage{amssymb}
\usepackage{mathtools}
\usepackage{amsthm}

\usepackage[capitalize,noabbrev]{cleveref}

\theoremstyle{plain}
\newtheorem{theorem}{Theorem}[section]

\theoremstyle{definition}

\theoremstyle{remark}

\usepackage[textsize=tiny]{todonotes}

\usepackage{csquotes}

\icmltitlerunning{Closing the Distribution Gap in Adversarial Training for LLMs}

\begin{document}

\twocolumn[
  \icmltitle{Closing the Distribution Gap in Adversarial Training for LLMs}





  \icmlsetsymbol{equal}{*}

  \begin{icmlauthorlist}
    \icmlauthor{Chengzhi Hu}{equal,tum}
    \icmlauthor{Jonas Dornbusch}{equal,tum}
    \icmlauthor{David Lüdke}{tum}
    \icmlauthor{Stephan Günnemann}{tum}
    \icmlauthor{Leo Schwinn}{tum}
  \end{icmlauthorlist}

    \icmlaffiliation{tum}{Technical University of Munich, Germany}

  \icmlcorrespondingauthor{Chengzhi Hu}{cm.hu@tum.de}
  \icmlcorrespondingauthor{Jonas Dornbusch}{jonas.dornbusch@tum.de}

  \icmlkeywords{Machine Learning, ICML}

  \vskip 0.3in
]



\printAffiliationsAndNotice{\icmlEqualContribution}
\newcommand{\DAT}{\textsc{DAT}}
\newcommand{\TV}{\mathrm{TV}}

\begin{abstract}
\looseness=-1
  Adversarial training for LLMs is one of the most promising methods to reliably improve robustness against adversaries.
   However, despite significant progress, models remain vulnerable to simple in-distribution exploits, such as rewriting prompts in the past tense or translating them into other languages.
  We argue that this persistent fragility stems from a fundamental limitation in current adversarial training algorithms: 
  they minimize adversarial loss on their training set but inadequately cover the data distribution, resulting in vulnerability to seemingly simple attacks. 
  To bridge this gap, we propose Distributional Adversarial Training, \DAT. 
  We leverage Diffusion LLMs to approximate the true joint distribution of prompts and responses, enabling generation of diverse, high-likelihood samples that address generalization failures. By combining optimization over the data distribution provided by the diffusion model with continuous adversarial training, DAT achieves substantially higher adversarial robustness than previous methods. Code and models are available on GitHub and Hugging Face \href{https://github.com/ASSELab/DAT}{(Link)}.
\end{abstract}

\section{Introduction}

Ensuring the safety of Large Language Models (LLMs) is a prerequisite for their reliable deployment. While alignment techniques such as RLHF~\cite{ouyang2022training} successfully reduce the likelihood of harmful outputs in normal use, models remain highly vulnerable to adversarial prompts~\cite{zou2023universal, zhu2023autodan, li2024llm}. 
Adversarial Training (AT) has emerged as one of the most promising defenses, aiming to improve robustness by augmenting training data with worst-case adversarial inputs~\cite{goodfellow_explaining_2015, madry_towards_2018}. Recently, AT has been successfully scaled to LLMs~\cite{xhonneux2024efficient, casper2024defending}. 
Yet, despite these advances, models remain surprisingly brittle, exhibiting simple \textit{generalization failures}. 
A model might refuse a complex, optimized attack string but effectively bypass its own safety training if the same request is merely translated into a low-resource language~\cite{yong2023low} or rephrased in the past tense~\cite{andriushchenko2024does}. 
These are not adversarial examples in the traditional sense of optimized \textit{model-specific} perturbations. Rather, they are valid, natural inputs that lie just outside the distribution observed during training.

\begin{figure}
    \centering
    \includegraphics[width=1\linewidth]{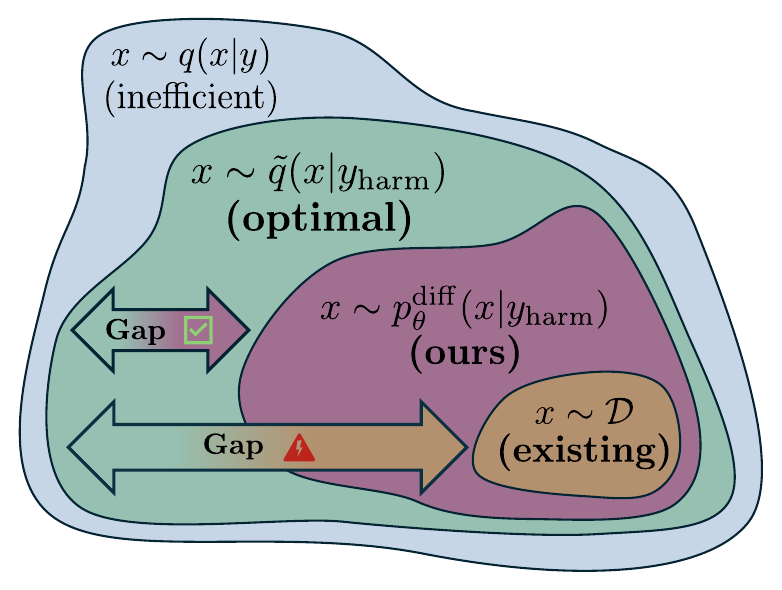}
   \caption{Standard AT minimizes the empirical robust risk over a fixed dataset $\mathcal{D}$ (brown), which provides a poor approximation of the population robust risk. This results in a distribution gap where the model remains vulnerable to the manifold of natural language $q$ (blue). Specifically, standard methods fail to cover the distribution of prompts $\tilde{q}(x\mid y_{harm})$ (green) that are likely to trigger harmful responses. Our DAT framework bridges this gap by optimizing over a surrogate distribution defined by a diffusion LLM $p^{\mathrm{diff}}_{\theta}(x | y_{\mathrm{hamr}})$ (purple), allowing the model to train on a distribution that more closely matches the true population.}
    \label{fig:hero_fig}
\end{figure}

We argue that these failures stem from the decomposition of the robust risk~\cite{madry_towards_2018}. Minimizing the robust risk over the true underlying data distribution $q(x,y)$ inherently involves minimizing two distinct sources of error: the \textit{data distribution approximation error} and the \textit{adversarial optimization error}. The former relates to how well the training data covers $q(x,y)$, while the latter relates to finding the worst-case perturbation within a local region. Standard AT algorithms effectively address the optimization error by robustifying models against \textit{model-specific} perturbations in a local ball. However, they neglect the approximation error, missing \textit{data-specific} vulnerabilities that naturally occur in $q(x,y)$ but lie outside the fixed training set (see Figure~\ref{fig:hero_fig}), such as a prompt reformulated to the past tense.

To address the approximation error, we propose optimizing over a generative surrogate, where we can directly sample data points $x$ from the high-likelihood region of harmful responses $q(x | y=\text{harmful})$. Diffusion LLMs~\cite{zhu2025llada15variancereducedpreference, ludke2025diffusion} are uniquely suited for this purpose. Unlike standard autoregressive models that only learn the conditional $q(y|x)$, diffusion models capture the \textit{joint distribution} of prompts and responses. This enables a tractable solution of the inverse problem: sampling diverse prompts $x$ conditioned on a specific harmful response $y$, thereby discovering data-specific vulnerabilities missing in the original data set.

In this work, we propose the \textit{Distributional Adversarial Training} (DAT) framework that simultaneously tackles both error sources. We leverage pretrained Diffusion LLMs to sample \textit{data-specific} adversarial examples that cover the support of $q$ and uncover vulnerabilities missed by fixed datasets. We then apply continuous AT to these samples to minimize the worst-case error, ensuring a good approximation of the local adversarial loss. This approach bridges the gap between data-specific and model-specific adversarial vulnerabilities. Our main contributions are:
\begin{itemize}
    \item \textbf{Formalization of the Robustness Gap:} We frame the failure of current defenses as a discrepancy between minimizing empirical and population-robust risk, distinguishing between model-specific perturbations and data-specific generalization failures.
    \item \textbf{Distributional Adversarial Training:} We introduce a novel training objective that combines generative sampling from Diffusion LLMs with continuous adversarial optimization, which simultaneously addresses both error sources in standard AT. Furthermore, we provide a theoretical result that, given a fidelity assumption on the surrogate function, optimizing over the surrogate distribution can reduce the gap between empirical and population-robust risk.
    \item \textbf{Improved Robustness:} Experimentally, we show that our DAT approach yields models that are robust against a variety of state-of-the-art adversarial attacks, considerably outperforming previous robustification methods that rely on static datasets. 
\end{itemize}
\section{Background}
\label{sec:background}

We briefly introduce notation regarding adversarial training and LLMs required to formalize our proposed method.

\subsection{Adversarial Training}
Adversarial Training (AT) formulates robustness as a min-max optimization problem. Let $x$ be some general input and $y$ the output of a probabilistic neural network $p_{\theta}(y \mid x)$ parametrized by $\theta$.
Further, let $\ell_{rob}(x, y; \theta) = \max_{\delta \in \Delta} \mathcal{L}(p_\theta(y \mid x+\delta))$ denote the robust loss, which represents the worst-case loss within a local perturbation set $\Delta$ (e.g., an $\epsilon$-ball in the embedding space). For a data distribution $q$ on input-output pairs, the goal is to minimize the \textit{population robust risk}~\cite{madry_towards_2018}
\begin{equation*}
    \mathcal{R}_{pop}(\theta) = \mathbb{E}_{(x,y) \sim q} [\ell_{rob}(x, y; \theta)].
\end{equation*}
In practice, we draw a finite dataset $\mathcal{D} \sim q^n$ and minimize the \textit{empirical robust risk}:
\begin{equation*}
    \mathcal{R}_{emp}(\theta) = \mathbb{E}_{(x,y) \sim \mathcal{D}} [\ell_{rob}(x, y; \theta)] \approx \mathcal{R}_{pop}(\theta).
\end{equation*}
Here, AT algorithms solve an inner maximization to approximate the supremum of the loss within $\Delta$, followed by an outer minimization to update $\theta$. While researchers have focused extensively on the inner loop, the gap between empirical and population robust risk remains unaddressed in current methods.

\subsection{Large Language Models}
Let $\mathcal{V}$ be a vocabulary and let $\mathcal{Z}=\mathcal{V}^{\le m}$ denote the set of token sequences of length at most $m$.
We denote the true distribution of language data as $q$ over $z\in\mathcal{Z}$.
Whenever convenient, we identify a sequence $z$ with a prompt--response pair $(x,y)$ via $z = x \oplus y$ where $x,y \in \mathcal{Z}$, and (by a slight abuse of notation) write $q(z)$ and $q(x,y)$ interchangeably for the corresponding probability mass function.
Autoregressive large language models (LLMs) aim to approximate the data-generating conditional $q(y\mid x)$ by modeling
\[
p^{\mathrm{AR}}_\theta(y \mid x) \;=\; \prod_{t} p^{\mathrm{AR}}_\theta\!\big(y_t \mid x, y_{<t}\big),
\]
for tokens $y_t \in \mathcal{V}$.
While such models are highly effective at conditional modeling, they do not explicitly parameterize the input marginal $q(x)$ and hence not the full joint $q(x,y)$.

\subsection{Adversarial Training for LLMs}
For classification problems, adversarial training is typically phrased as enforcing \emph{label invariance} under small perturbations of the input: for a fixed ground-truth label $y$, the learner should maintain its prediction when $x$ is replaced by $x+\delta$.
In safety-oriented robustness for large language models, another definition is used~\cite{xhonneux2024efficient, sheshadri2024latent}.
In the presence of jailbreak-style attacks, the central requirement is \emph{output safety}: the model should avoid emitting harmful content, regardless of which prompt happens to elicit it.
Accordingly, it is natural to view prompts $x$ primarily as \emph{triggers} for undesirable responses, rather than as inputs whose semantics must be preserved under perturbations.

We model harmfulness through an indicator function $h:\mathcal{Z}\to\{0,1\}$, where $h(y)=1$ denotes that a sequence $y$ contains harmful content.
To make this perspective explicit, we define the adversarial-training distribution as the natural language distribution restricted to harmful outputs,
\begin{equation*}
    \tilde{q}(x,y)
    \;:=\;
    q\bigl(x,y \mid h(y)=1\bigr)
    \;=\;
    \frac{q(x,y)\,\mathbf 1\{h(y)=1\}}{q\bigl(h(y)=1\bigr)}.
\end{equation*}
By construction, this restriction preserves the relative frequencies of harmful responses and their associated prompts as they occur under $q$ (in particular, $\tilde q(y)=q(y\mid h(y)=1)$).
Moreover, for any harmful response $y$ with $h(y)=1$ (i.e., whenever $\tilde q(y)>0$), the induced prompt conditional is unchanged,
\begin{equation}\label{eq:harmful_part}
    \tilde q(x\mid y)
    \;=\;
    \frac{\tilde q(x,y)}{\tilde q(y)}
    \;=\;
    \frac{q(x,y)}{q(y)}
    \;=\;
    q(x\mid y).
\end{equation}
Consequently, sampling $(x,y)\sim \tilde{q}$ admits a two-stage interpretation:
first draw a harmful response $y\sim q(\cdot\mid h(y)=1)$, and then draw a prompt $x\sim q(\cdot\mid y)$.
In practice, the first stage can be implemented by sampling $y$ from a dataset of harmful responses, while the second stage can be approximated by a conditional generator for $x\mid y$.

Crucially, this choice prioritizes prompts that are \emph{naturally compatible} with a harmful response under $q$.
With limited adversarial-training budget, focusing on such high-probability triggers is desirable, whereas spending capacity on prompts that are exceedingly unlikely to elicit harmful outputs is less informative for improving real-world jailbreak robustness.

\subsection{Model- and Data-Specific Prompts}\label{sec:data_specific}

We call a prompt $x$ \textit{data-specific} if it has high likelihood under the harmful prompt marginal $\tilde q(x)$, i.e., if it is likely to elicit a harmful response under the underlying data distribution. Since different models are generally trained on large fractions of the same data, such prompts tend to be transferable across models:
\begin{equation}\label{eq:data_specific}
\tilde q(x) \approx \tilde p_{\theta_1}(x) \approx \cdots \approx \tilde p_{\theta_N}(x),
\end{equation}
with the marginal $
\tilde p_{\theta}(x)
\coloneqq
\sum_{y \in \mathcal{Z}:\, h(y)=1} p_{\theta}(x,y).
$

We contrast this to \textit{model-specific} prompts, which we define as prompts that have a low likelihood to trigger harmful responses under the true harmfulness distribution but a high likelihood to trigger harmful responses for individual models $\tilde{q}(x) < \tilde p_{\theta_1}(x)$. 
Accordingly, these prompts generally do not transfer between models $\tilde p_{\theta_1}(x) \gg \tilde p_{\theta_2}(x)$. This includes most attacks, such as GCG or BoN (see Figure~\ref{fig:transfer_plot}).

\section{Method}
\label{sec:method}

We introduce Distributional Adversarial Training (\DAT), a framework that leverages generative surrogate models to address a generalization gap in previous AT methods by better approximating the population risk.

\subsection{Adversarial Training Suffers from a Generalization Gap}
We argue that current adversarial training approaches for LLMs fail to generalize because the empirical risk $\mathcal{R}_{emp}$ insufficiently approximates the population risk $\mathcal{R}_{pop}$ for two main reasons. First, the diversity of open-ended natural language, which is difficult to cover with small, finite datasets $\mathcal{D}$. Secondly, the high sample complexity required for robust learning~\citep{schmidt2018adversarially}. Consequently, models remain vulnerable to simple exploits like past-tense rephrasing~\cite{andriushchenko2024does} or translation~\cite{yong2023low}, outside the training distribution.

\subsection{Generative Surrogates}
To bridge the generalization gap, we propose to replace the static dataset $\mathcal{D}$ with a generative surrogate $p_{\theta}(x, y)$ that allows us to better approximate the population risk over the distribution of harmful inputs $\tilde{q}(x,y)$.
A suitable surrogate must therefore satisfy three key desiderata.

\textbf{Conditional Sampling.} First, the surrogate must be able to effectively invert the generation process, to sample data $x$ conditioned on a harmful response $y$, as in Equation~\ref{eq:harmful_part}.

\textbf{Data Specific.} Second, sampled prompts should be \textit{data specific} as described in Equation~\ref{eq:data_specific}.

\textbf{Diversity.} Third, it must cover the diverse support of natural language, preventing the training from overfitting to a narrow set of attack patterns.

\begin{figure}[t]
    \centering
    \includegraphics{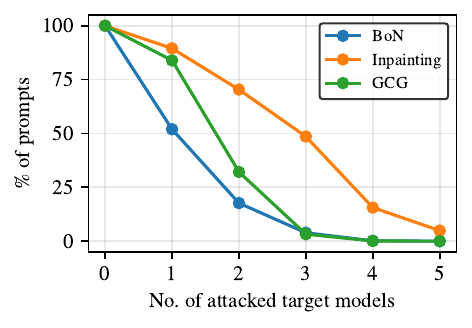}
    \caption{Cumulative transfer ASR across five target models (Gemma3-12B \cite{team2025gemma}, Qwen2.5-7B\cite{qwen2025qwen25technicalreport}, Zephyr-7B\cite{tunstall2023zephyr}, Llama3-8B-LAT \cite{sheshadri2024latent}, Llama3-8B-CB \cite{zou2024circuitbreaker}) from attacks on Llama3-8B. Diffusion-based Inpainting attacks exhibit significantly higher transferability than model-specific optimization (GCG) or heuristic perturbations (BoN), suggesting that conditional sampling from the diffusion surrogate effectively identifies data-specific vulnerabilities that generalize across architectures and defenses.}
    \label{fig:transfer_plot}
    \vspace{-0.8cm}
\end{figure}

\subsection{Diffusion LLMs as Surrogate}
We employ Diffusion LLMs as our generative surrogate because they satisfy our technical desiderata for Distributional Adversarial Training. 
Diffusion models have recently emerged as a powerful alternative for text generation, modeling the joint distribution $q(x,y)$ rather than the standard autoregressive conditional $q(y|x)$~\citep{nie2025largelanguagediffusionmodels}. Crucially, this enables sampling from the conditional distribution $q(x|y)$. Recently, ~\citet{ludke2025diffusion} demonstrated that fixing a target response $y$ and performing inpainting-like conditioning, diffusion LLMs can effectively sample adversarial prompts $x \sim p_{\theta}^{\mathrm{diff}}(x|y)$ directly from the learned posterior, fulfilling our first requirement.

Moreover, as shown in \cref{fig:transfer_plot}, prompts sampled via diffusion exhibit significantly higher transferability than heuristic attacks such as GCG or BoN, confirming that they capture data-specific properties of the true distribution and thus fulfill the data specificity requirement.

Lastly, diffusion models fulfill our diversity criteria. An alternative approach to approximate the posterior $\tilde{q}(x|y)$ would be to use discrete adversarial attacks (e.g., GCG). However, we show in \cref{fig:bert_diversity} that diffusion-based sampling generates substantially more diverse samples, ensuring broader coverage of the support.


\begin{figure}[t]
    \centering
    \includegraphics{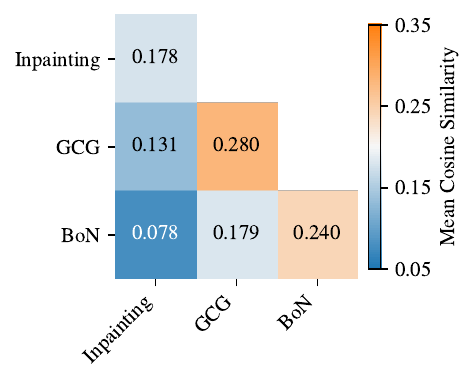}
    \caption{Diversity of generated attack strings measured using SBERT embeddings (all-MiniLM-L6-v2; \citealp{reimers-gurevych-2019-sentence}). Each cell reports the mean pairwise cosine similarity between samples generated by two methods. Diffusion-based attacks exhibit the lowest intra-method similarity (0.178), indicating substantially greater sample diversity than GCG and BoN.}
    \label{fig:bert_diversity}
\end{figure}

\subsection{Tractability via Monte Carlo Sampling}
While the ideal objective is to minimize the expected worst-case loss over the full diffusion distribution conditioned on harmful targets $p^{\mathrm{diff}}_{\theta}(x,y | h(y)=1)$, direct optimization is intractable. We resolve this by approximating the expectation via Monte Carlo sampling. In practice, we leverage an initial dataset of harmful pairs $(x, y)$. We fix the affirmative harmful responses $y$ from this dataset, creating $\mathcal{D}_{\mathrm{harm}}$ of harmful responses, and use the diffusion model to \textit{conditionally generate} diverse, new data samples $\tilde{x} \sim p^{\mathrm{diff}}_{\theta}(x \mid y)$ using the method proposed in~\citep{ludke2025diffusion}. This effectively expands the training set by generating many variations of prompts for each harmful target. In practice, we additionally filter the generated samples according to their effectiveness $p_{\theta}^{\mathrm{AR}}(y \mid x)$, keeping only those that trigger harmful responses with high likelihood.

\subsection{Distributional Adversarial Training}
Our proposed Distributional Adversarial Training (\textit{DAT}) unifies these concepts. We optimize the following objective:
\begin{equation}
\begin{aligned}
\mathcal{L}_{DAT} = &\underbrace{\mathbb{E}_{y \sim \mathcal{D}_{\mathrm{harm}}, x \sim p^{\mathrm{diff}}_{\theta}(x|y)}}_{\text{Monte Carlo Sampling}} \\
&\left[ \underbrace{\max_{\delta \in \Delta} \mathcal{L}_{adv}(x+\delta, y, y_{refusal})}_{\text{Adversarial Perturbation}} \right] + \lambda_{\mathrm{ret}}\mathcal{L}_{KL}
\end{aligned}
\end{equation}
where we use continuous adversarial training (CAT) to compute the adversarial loss $\mathcal{L}_{adv}$ \citep{xhonneux2024efficient}. Specifically, we solve the \textbf{inner loop} maximization by applying continuous perturbations $\delta$ to the token embeddings:
\begin{equation}
    \mathcal{L}_{adv} = \underbrace{\log p_\theta(y \mid x+\delta)}_{\text{Away Loss}} - \underbrace{\log p_\theta(y_{refusal} \mid x+\delta)}_{\text{Toward Loss}}.
\end{equation}
Here, $y$ is the desired affirmative harmful response and $y_{refusal}$ is the model's safe refusal. The \textbf{outer loop} minimization then penalizes the model for these discovered vulnerabilities. To prevent the model from collapsing to degenerate solutions and to maintain general performance, we regularize the outer minimization with a KL divergence term on a retain set $\mathcal{D}_{\mathrm{ret}}$ weighted by a scalar $\lambda_{KL}\in\mathbb{R}^{+}$:
\begin{equation}
    \mathcal{L}_{KL} = \mathbb{E}_{x \sim \mathcal{D}_{\mathrm{ret}}} [D_{\mathrm{KL}}(p_{\theta_0}(\cdot \mid x) \parallel p_\theta(\cdot \mid x))],
\end{equation}
where $p_{\theta_0}$ is the model before adversarial training. This unified objective ensures that the model is robust against diverse, high-likelihood data samples discovered in the outer loop while improving robustness against adversarial attacks in the inner loop.

\subsection{Theoretical Justification}
We formally justify the use of the diffusion surrogate to approximate the population risk. Let $\mathcal{R}_{pop}(\theta) = \mathbb{E}_{(x,y) \sim q}[\ell_{rob}(x,y;\theta)]$ be the true population risk, where $\ell_{rob}$ is the inner robust loss. We show that optimizing our surrogate objective $\mathcal{R}_{diff}$ bounds the population risk.

\begin{theorem}[Surrogate Fidelity Bound]
\label{thm:fidelity}
Assume the robust loss is bounded, i.e., $|\ell_{rob}(x,y;\theta)| \le M$ for all $(x,y)\in\mathcal{Z}$, and the diffusion surrogate satisfies the conditional fidelity assumption
\[
\mathbb{E}_{y\sim \tilde q(y)}
\left[
\TV\!\bigl(q(\cdot\mid y),p^{\mathrm{diff}}_\theta(\cdot\mid y)\bigr)
\right]\le \varepsilon .
\]
Then
\[
|\mathcal{R}_{pop}(\theta) - \mathcal{R}_{diff}(\theta)| \le 2 M \varepsilon .
\]
\end{theorem}
This result (proof in \cref{app:proof_theorem1}) guarantees that improving the generative fidelity of the diffusion model directly translates to a tighter approximation of the true robust generalization risk, unlike standard AT which is limited by the finite size of $\mathcal{D}$.

\section{Experiment Setup}\label{sec:exp_setup}

We performed all experiments on a cluster of A100, H100, and H200 nodes.

\textbf{Models.}
We study two instruction-tuned LLMs: \textsc{Llama3-8B-Instruct} ~\citep{grattafiori2024llama}, \textsc{Qwen2.5-14B-Instruct}~\citep{qwen2025qwen25technicalreport}. 
For brevity, we refer to the models as Llama3-8B and Qwen2.5-14B when no ambiguity arises.
Unless stated otherwise, we run most ablations on Llama3-8B and report multi-model results in \cref{tab:main_result}. For more information see Appendix~\ref{app:Models}.


\textbf{Training Data.}
\label{sec: training_data}
We use the \textsc{HarmBench} adversarial training split~\citep{mazeika2024harmbench} as the source of harmful behaviors.
For each behavior, we generate candidate harmful prompts via diffusion inpainting following~\citet{ludke2025diffusion}: we condition the pretrained diffusion LLM \textsc{LLaDA-8B-Base}~\cite{nie2025largelanguagediffusionmodels} (LLaDA) on an affirmative harmful target and sample $1000$ prompt variants.
Because conditional samples can vary in linguistic quality and effectiveness, we filter these candidates by evaluating each prompt on the target model and scoring the resulting completion with the StrongREJECT rubric judge~\citep{souly2024strongreject}.
For the main experiments, we keep the top $16$ prompts per behavior, yielding $1600$ diffusion-generated harmful prompts in total.
We study the effect of dataset size and filtering in \cref{sec:training_data_size} and \cref{sec:training_data_quality}.
As in prior work~\cite{xhonneux2024efficient}, we use UltraChat200k as a retain set to preserve utility.

\textbf{DAT Training.}
We perform \DAT{} by applying Continuous Adversarial Training (CAT)~\citep{xhonneux2024efficient} to the diffusion-generated prompts.
Based on preliminary findings in our experiments that they improve the utility-robustness trade-off, we slightly adapt the original CAT algorithm. We replace the cross-entropy retain loss with a KL divergence term, which was found to improve utility in~\cite{sheshadri2024latent}. Additionally, we remove loss thresholds, which we found to improve robustness. These changes are applied to both \DAT{} and the CAT baseline.
Regarding the adversarial constraints, we set the $\epsilon$-ball radius to infinity, limiting attack strength solely by the number of iterations to reduce the number of tunable hyperparameters.
Finally, to ensure a fair comparison, we align the total number of parameter updates and optimizer settings across all trained models. We selected models for full robustness evaluations based on proxy evaluations similar to those described in~\cite{beyer2025fast}.
All hyperparameters are tuned on the CAT baseline and transferred to \DAT{}. See also \cref{app:Training}.

\textbf{Baselines.}
We compare \DAT{} against the original (non-adversarially trained) models, standard adversarial training methods (CAT~\citep{xhonneux2024efficient}, LAT~\citep{sheshadri2024latent}), hybrid approaches that combine discrete and continuous attacks (MixAT-GCG~\citep{dekany2025mixat}), and Circuit Breakers (CB;~\citep{zou2024improving}).

\textbf{Evaluation.}
We evaluate robustness on 100 malicious behaviors from \textsc{JailbreakBench}~\citep{chao2024jailbreakbench} using a suite of model-specific and data-specific attacks using the implementations provided in AdversariaLLM~\cite{beyer2025adversariallm} to improve reproducibility~\cite{schwinn2025adversarial,beyer2025llm}. This includes: GCG~\citep{zou2023universal}, PAIR~\citep{chao2023jailbreakingpair}, BoN~\citep{hughes2024best}, and diffusion Inpainting~\citep{ludke2025diffusion}. Furthermore, we also evaluate robustness against sampling-based attacks using the original prompt, which we refer to as Direct Attack~\cite{beyer2025sampling, scholten2025a}.
For each behavior–attack pair, we generate multiple completions and score every output with the StrongREJECT judge~\citep{souly2024strongreject}, which assigns a harmfulness score $\mathcal{H} \in [0,1]$ to the prompt–response pair.
We classify an output as harmful if $\mathcal{H}$ exceeds a fixed threshold of $0.5$ and mark the attack as successful.
To capture worst-case robustness across attacks we also report a Best-of-All (BoA) ASR, analogous to the ALO-ASR metric of MixAT~\citep{dekany2025mixat}, which marks a behavior as broken if any attack succeeds. We provide a list of attack hyperparameters in Appendix~\ref{app:Attacks}.
Finally, we assess utility using XSTest~\citep{rottger-etal-2024-xstest} for helpfulness, and MMLU~\citep{hendrycks2021measuringmassivemultitasklanguage}, ARC-E, and ARC-C~\citep{allenai:arc} for general capabilities, following their standard evaluation protocols.


\newcommand{\datASRCell}[1]{%
  \begingroup
  \ifdim #1pt < 0.250001pt \cellcolor{green!18}\else
    \ifdim #1pt < 0.400001pt \cellcolor{green!10}\else
      \ifdim #1pt < 0.550001pt \cellcolor{green!4}\else
        \ifdim #1pt < 0.700001pt \cellcolor{red!4}\else
          \ifdim #1pt < 0.850001pt \cellcolor{red!10}\else
            \cellcolor{red!18}%
          \fi
        \fi
      \fi
    \fi
  \fi
  \num[round-mode=places,round-precision=2,minimum-decimal-digits=2]{#1}%
  \endgroup
}
\newcommand{\datUtilCell}[1]{%
  \begingroup
  \dimen0=\dimexpr 1pt - #1pt\relax
  \ifdim \dimen0 < 0.250001pt \cellcolor{green!18}\else
    \ifdim \dimen0 < 0.400001pt \cellcolor{green!10}\else
      \ifdim \dimen0 < 0.550001pt \cellcolor{green!4}\else
        \ifdim \dimen0 < 0.700001pt \cellcolor{red!4}\else
          \ifdim \dimen0 < 0.850001pt \cellcolor{red!10}\else
            \cellcolor{red!18}%
          \fi
        \fi
      \fi
    \fi
  \fi
  \num{#1}%
  \endgroup
}
\newcommand{\asr}[1]{\datASRCell{#1}}
\newcommand{\util}[1]{\datUtilCell{#1}}
\newcommand{\na}{\multicolumn{1}{c}{--}}

\begin{table*}[t]
\centering
\small
\setlength{\tabcolsep}{3.0pt}
\caption{{Main results on Llama3-8B and Qwen2.5-14B across robustness and utility benchmarks. Lower ASR indicates stronger robustness, while higher utility scores indicate better model helpfulness and capabilities.}
}
\label{tab:main_result}
\resizebox{0.85\textwidth}{!}{%
\begin{tabular}{l l *{6}{c} *{4}{c}}
\toprule

\multirow{2}{*}{Model} &
\multirow{2}{*}{Method} &
\multicolumn{6}{c}{\textbf{ASR $\downarrow$}} &
\multicolumn{4}{c}{\textbf{Utility $\uparrow$}} \\
\cmidrule(lr){3-8}\cmidrule(lr){9-12}
& &
{GCG} & {Pair} & {BoN} & {Direct} & {Inpainting} & {BoA}
& {XSTest} & {MMLU} & {ARCe} & {ARCc} \\
\midrule
\multirow{9}{*}{\rotatebox{90}{Llama3-8B}}
& Default & \asr{0.39} & \asr{0.3} & \asr{0.99} & \asr{0.33} & \asr{0.98} & \asr{1} & \util{0.660} & \util{0.638} & \util{0.815} & \util{0.530} \\
\cmidrule(lr){2-12}
& CAT & \asr{0.04} & \asr{0.13} & \textbf{\asr{0}} & \textbf{\asr{0}} & \asr{0.88} & \asr{0.88} & \util{0.464} & \util{0.637} & \util{0.810} & \util{0.520} \\
& LAT & \asr{0.09} & \asr{0.15} & \textbf{\asr{0}} & \asr{0.01} & \asr{0.94} & \asr{0.94} & \util{0.444} & \util{0.619} & \util{0.784} & \util{0.485} \\
& CB & \textbf{\asr{0.03}} & \asr{0.09} & \asr{0.56} & \asr{0.13} & \asr{0.92} & \asr{0.94} & \textbf{\util{0.672}} & \util{0.636} & \util{0.815} & \util{0.533} \\
& MixAT-GCG & \asr{0.04} & \textbf{\asr{0.03}} & \asr{0.02} & \asr{0.01} & \asr{0.77} & \asr{0.77} & \util{0.456} & \util{0.612} & \textbf{\util{0.827}} & \util{0.513} \\
\cmidrule(lr){2-12}
&\shortstack[l]{Diffusion-only} & \asr{0.2} & \asr{0.12821} & \asr{0.875} & \asr{0.04} & \textbf{\asr{0.23}} & \asr{0.88} & \util{0.536} & \textbf{\util{0.644}} & \util{0.823} & \textbf{\util{0.536}} \\
& \textbf{DAT} & \asr{0.08} & \asr{0.08} & \textbf{\asr{0}} & \asr{0.02} & \asr{0.32} & \textbf{\asr{0.36}} & \util{0.464} & \util{0.639} & \util{0.813} & \util{0.521} \\
\midrule
\multirow{4}{*}{\rotatebox{90}{Qwen2.5-14B}}
& Default & \asr{0.92} & \asr{0.45} & \asr{0.98} & \asr{0.27} & \asr{0.99} & \asr{1} & \util{0.748} & \textbf{\util{0.789}} & \textbf{\util{0.862}} & \textbf{\util{0.619}} \\
\cmidrule(lr){2-12}
& CAT & \asr{0.09} & \asr{0.14} & \asr{0.11} & \asr{0} & \asr{0.92} & \asr{0.93} & \textbf{\util{0.544}} & \util{0.787} & \util{0.856} & \util{0.607} \\
& MixAT-GCG & \asr{0.05} & \textbf{\asr{0.02}} & \textbf{\asr{0}} & \textbf{\asr{0}} & \asr{0.75} & \asr{0.75} & \util{0.388} & \util{0.778} & \textbf{\util{0.862}} & \textbf{\util{0.619}} \\
\cmidrule(lr){2-12}
& \textbf{DAT} & \textbf{\asr{0.05}} & \asr{0} & \asr{0.02} & \textbf{\asr{0.01}} & \textbf{\asr{0.14}} & \textbf{\asr{0.18}} & \util{0.464} & \util{0.773} & \util{0.816} & \util{0.584} \\
\bottomrule
\end{tabular}%
}
\end{table*}

\section{Results}

Our experiments evaluate if \DAT{} closes the distribution gap identified in \cref{sec:method}. 
We conduct a series of experiments to assess: \textbf{A)} if \DAT{} improves worst-case robustness
against model-specific and data-specific attacks (\S\ref{sec:robustness}), \textbf{B)} the importance of continuous adversarial training in our framework (\S\ref{sec:adversarial_ablation}), \textbf{C)} if \DAT{} achieves Pareto-optimal trade-offs between robustness and utility (\S\ref{sec:pareto}), \textbf{D)} if we can
empirically validate our theoretical argument that better approximating the data distribution via increased diffusion sampling leads to improved robustness (\S\ref{sec:training_data_size}), and \textbf{E)} if it is necessary to train on \textit{data specific} prompts or this property is not required to minimize the robust risk (\S\ref{sec:training_data_quality}).

\subsection{\DAT{} Considerably Improves Worst-Case Robustness}\label{sec:robustness}

We present our main results in Table \ref{tab:main_result}. 

\textbf{Robustness.} While baseline defenses (CAT, LAT) and CB achieve low ASR on gradient-based and optimization attacks like GCG, Pair, they exhibit considerable generalization failures. Notably, CB remains vulnerable to BoN, reaching an ASR of $56\%$. Furthermore, the \textit{Inpainting} attack breaks all prior defenses. For Llama3-8B, Inpainting achieves ASRs ranging from $77\%$ (MixAT) to $94\%$ (LAT). Even MixAT, which explicitly combines different attack strategies during training, fails to defend against Inpainting. We hypothesize that MixAT cannot close this generalization gap because its training augmentation relies on model-specific attacks such as GCG rather than addressing data-specific vulnerabilities like \DAT. As we are generally interested in a lower bound on robustness, we also computed the BoA ASR, which provides a lower bound on robustness against an ensemble of all attacks. Here, DAT substantially outperforms all other approaches, achieving an ASR of $36\%$ for the Llama model compared to the second-best method, MixAT ($77\%$). Results for Qwen2.5-14B are similar. Here, \DAT{} reduces the BoA ASR to $18\%$, whereas all other adversarial training approaches remain vulnerable, ranging from $75\%$ to $93\%$.

\textbf{Utility.} 
We evaluate utility across multiple benchmarks: XSTest for helpfulness, and MMLU, ARC-E, and ARC-C for general capabilities.
On Llama3-8B, \DAT{} maintains competitive helpfulness ($0.464$), with good utility across all metrics, performing comparably to other adversarial training baselines such as CAT.
While CB achieves a higher XSTest score ($0.672$), we note that its training process explicitly incorporates the XSTest dataset. Consequently, this result likely reflects training performance rather than generalization, rendering a direct comparison with the other methods for this specific metric non-meaningful. On Qwen2.5-14B, \DAT{} maintains reasonable utility while achieving substantially higher robustness than all baselines.
These results demonstrate that \DAT{} achieves a favorable robustness--utility trade-off across diverse model sizes and architectures.

\subsection{Good Approximation of the Adversarial Loss Improves Robustness}\label{sec:adversarial_ablation}

Next, we analyze a Diffusion-only ablation, using the diffusion model as a surrogate for the data distribution $q(x,y)$ without continuous adversarial optimization on the Llama3-8B model. Results are shown in Table~\ref{tab:main_result} under the method name Diffusion-only. The trained model achieves remarkable robustness against Inpainting ($23\%$ ASR) compared to previous approaches. 
However, the Diffusion-only model remains vulnerable to other adversarial attacks (e.g., GCG, BoN). In contrast, \DAT{} achieves high robustness against both attack types. This demonstrates that minimizing the total robust risk requires a dual approach: a valid approximation of the population risk in the outer loop of adversarial training, and of the worst-case adversarial loss for each data point in the inner loop.


\subsection{\DAT{} Exhibits Pareto-Optimal Robust--Utility Trade-Offs}\label{sec:pareto}

For a defense to be practical it must preserve the model's helpfulness and compliance on benign requests. 
We evaluate this trade-off on Llama3-8B by varying the regularization strength $\lambda_{KL}$ (balancing the robust loss and the KL-divergence on the retain set), and the number of performed attack iterations in the inner loop of adversarial training. This allows us to trace the Pareto frontier between \textit{Inpainting Robustness} ($1 - \text{ASR}$) and \textit{XSTest Compliance}. 
As illustrated in \Cref{fig:pareto_tradeoff}, \DAT{} consistently expands the Pareto-front across all hyperparameter configurations. 
Notably, as we tune for higher compliance and utility, \DAT{} maintains substantially greater robustness than baseline adversarial training methods such as CAT and MixAT-GCG. 
For instance, even for compliance levels close to that of the original model, \DAT{} remains as robust as previous adversarial training approaches. 
These results show that \DAT{} maintains a higher level of robustness across various compliance targets compared to baseline approaches, which can be effectively controlled with interpretable hyperparameters.

\begin{figure}[t]
    \centering
    \includegraphics{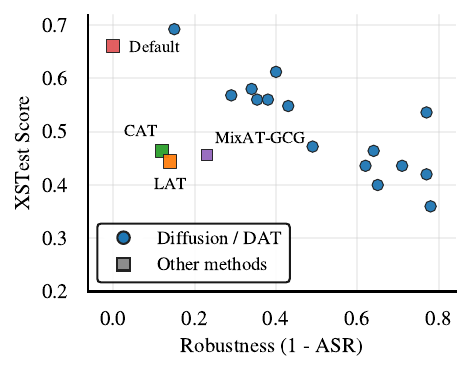}
    \caption{Pareto frontier for Llama3-8B showing the trade-off between Inpainting Robustness ($1-\text{ASR}$) and XSTest compliance rate. \DAT{} achieves superior trade-offs across all hyperparameter settings.}
    \label{fig:pareto_tradeoff}
    \vspace{-0.5cm}
\end{figure}

\subsection{Better Approximating the Data Distribution Improves Robustness}
\label{sec:training_data_size}

A core motivation of \DAT{} is that expanding the empirical training distribution via a generative surrogate reduces the approximation error inherent in robust learning. 
To verify this, we investigate how the number of unique diffusion-generated prompts per behavior affects the final model robustness.
As shown in \Cref{fig:data_size}, increasing the number of samples $M$ in our Monte Carlo approximation results in a consistent reduction in the Inpainting ASR, which is the strongest attack in our experiments. 
Specifically, as the number of unique prompts per harmful behavior in training data increases from $1$ to $16$, the ASR on Llama3-8B decreases from $54\%$ to $22\%$.
This improvement provides empirical evidence for \cref{thm:fidelity}: by more accurately approximating the conditional distribution $q(x|y)$ with the diffusion surrogate $p^{\mathrm{diff}}_\theta$, we effectively close the gap between empirical and population-robust risk.
Notably, this scaling of harmful data does not compromise the model's helpfulness. Helpfulness scores on XSTest remain stable around 0.4 across all sample sizes.
\begin{figure}[t]
    \centering
    \includegraphics{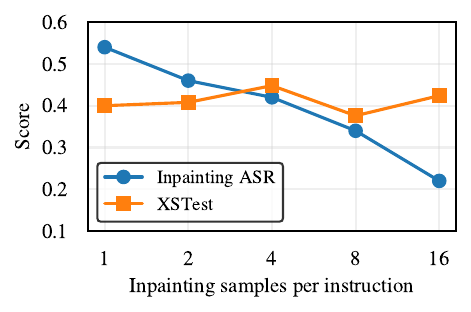}
    \caption{Inpainting ASR as a function of the number of unique diffusion samples $M$ per behavior in the training data. Robustness gradually improves as we better approximate the population risk $q(x,y)$, while helpfulness stays consistent.}
    \label{fig:data_size}
\end{figure}

\subsection{Data Specificity is Required for Robustness}
\label{sec:training_data_quality}

We further evaluate whether the observed robustness gains are simply due to increased data volume or whether the diffusion surrogate's ability to sample from the region of high-likelihood harmful prompts $\tilde{q}(x \mid y)$ is essential.
We conduct an ablation where we vary the quality of generated samples by filtering them based on the conditional likelihood $p^{\mathrm{AR}}_\theta(y \mid x)$ under the target model. 


We follow the data selection approach in \cref{sec: training_data} by training two models: one using the $16$ best samples per behavior (same as \DAT{} in \Cref{tab:main_result}) and another using the $16$ worst samples selected from $1000$ generations from the diffusion LLM. 
When normalizing for equal compliance, training on low-likelihood samples, those that fail to elicit the target harmful response $y$ even from an undefended model, yields limited robustness ($39\%$). 
In contrast, high-likelihood samples discovered by $p^{\mathrm{diff}}_\theta$ are highly effective, achieving $78\%$ robustness. Notably, using even a single high-likelihood sample per behavior achieves higher robustness ($46\%$) than using the set of low-likelihood samples (see \Cref{fig:data_size}).
Our results provide evidence that the surrogate must faithfully capture the harmful part of the distribution $\tilde{q}$; samples that fall outside the high-likelihood manifold of natural attacks do not contribute to minimizing the robust risk effectively. Baselines like MixAT rely on model-specific optimization that diverges from the natural manifold, explaining their inability to generalize to data-specific vulnerabilities such as inpainting.

\section{Related Work}

\paragraph{Adversarial Training in LLMs}
Adversarial training has been one of the most effective methods to improve robustness in deep learning \citep{szegedy_intriguing_2014, goodfellow_explaining_2015, madry_towards_2018, shafahi_adversarial_2019, wong_fast_2019, kireev2022effectiveness}. However, its application to LLMs faces significant computational challenges due to the discrete nature of text.
Early approaches often relied on computationally expensive discrete optimization methods to generate adversarial examples during training \citep{mazeika2024harmbench}.
To address this, more efficient methods operating in the continuous embedding space or latent representations have been proposed.
For instance, \citet{xhonneux2024efficient} introduced an efficient AT framework using continuous embedding attacks and demonstrated that this approach considerably improves robustness against discrete attacks while being orders of magnitude more efficient.
Other works focus on the model's internal representations: \citet{zou2024improving} proposed short-circuiting to remap harmful representations to refusal states; \citet{sheshadri2024latent} and \citet{casper2024defending} introduced latent adversarial training, which perturbs latent activations to trigger and mitigate failure modes; and \citet{yu2025robust} proposed Refusal Feature Adversarial Training (ReFAT), which ablates specific refusal features in the residual stream.
Additionally, \citet{liu2024adversarial} proposed a two-stage adversarial tuning framework involving token-level and out-of-distribution adversarial prompt generation.

A recent relevant approach is MixAT \citep{dekany2025mixat}, which combines model-specific discrete and continuous attacks to improve robustness. While MixAT focuses on better approximating the local worst-case adversary through this combination, our approach instead focuses on better approximating the population risk.
Lastly, some classical approaches have used generative models successfully in the context of adversarial defense.
For example, \citet{liu_gandef_2019} proposed GanDef, a GAN-based defense that utilizes an adversarial game to regularize feature selection for improved robustness. \citet{wang2023better} use an unconditional diffusion model to augment the training data, considerably improving robustness.

\paragraph{Adversarial Attacks in LLMs}
Adversarial attacks on LLMs have evolved rapidly, ranging from optimization-based methods to automated generation techniques.
Optimization-based attacks like GCG \citep{zou2023universal}, soft prompts~\citep{schwinn2023adversarial, schwinn2024soft} and PGD \citep{geisler2024attacking} use gradient information to find adversarial suffixes or continuous perturbations.
Automated methods such as AutoDAN \citep{zhu2023autodan}, Open Sesame~\citep{lapid2023open}, and PAIR \citep{chao2023jailbreakingpair} leverage attacker LLMs or genetic algorithms to generate interpretable jailbreaks.
Other notable approaches include purely random and model-unspecific attacks such as Best-of-N jailbreaking \citep{hughes2024best}, or simple adaptive attacks \citep{andriushchenko2024jailbreaking} which combine human manual tuning with optimization.
Recently, \citet{geisler2025reinforce} proposed using REINFORCE to optimize for expected harmfulness, addressing limitations in fixed-target objectives.
Some of these attack methods have been integrated into standardized evaluation frameworks like HarmBench \citep{mazeika2024harmbench} and used for adversarial training in works like MixAT \citep{dekany2025mixat}. Finally, \citet{ludke2025diffusion} demonstrate that diffusion LLMs can be used to invert the conditional distribution $p(y \mid x)$ to sample diverse, high-likelihood adversarial prompts from specific manifold regions that target harmful response $y$ without requiring gradient-based optimization. We use this approach to better approximate the population risk in adversarial training.


\section{Conclusion}

\textbf{Limitations.} 
While \DAT{} demonstrates the efficacy of generative surrogates, it currently relies on conditional sampling from a specific class of diffusion models. 
Future work could explore alternative strategies for approximating the data distribution, such as unconditional sampling from generative models, a technique proven effective in the image domain~\citep{wang2023better}.

In this work, we observe that while most adversarial training approaches focus on the inner loop of computing effective adversaries, they often ignore the potential limitations inherent in approximating the population risk via finite empirical datasets. 
To address this, we propose Distributional Adversarial Training (\DAT), a theoretically grounded framework that leverages Diffusion LLMs as generative surrogates to better approximate this risk. 
By actively sampling diverse, data-specific adversarial prompts from the joint distribution, \DAT{} effectively closes the generalization gap that limits existing methods. 
Empirically, our approach achieves substantially higher worst-case robustness against both model-specific and data-specific attacks, while consistently maintaining superior utility-robustness trade-offs compared to state-of-the-art baselines.

\clearpage
\section*{Impact Statement}

This work introduces Distributional Adversarial Training (DAT), a method designed to improve the reliability of Large Language Models (LLMs) against adversarial attacks. The primary impact of this research is on the area of AI safety. By improving the robustness of models against attacks, DAT reduces the likelihood of models being manipulated into generating harmful content. We do not identify specific negative impacts of our work, which we feel must be specifically highlighted here.

\bibliography{example_paper}

@inproceedings{goodfellow_explaining_2015,
	title = {Explaining and harnessing adversarial examples},
	booktitle = {ICLR},
	author = {Goodfellow, Ian J. and Shlens, Jonathon and Szegedy, Christian},
	year = {2015},
}

@inproceedings{madry_towards_2018,
	title = {Towards {Deep} {Learning} {Models} {Resistant} to {Adversarial} {Attacks}},
	booktitle = {ICLR},
	author = {Madry, Aleksander and Makelov, Aleksandar and Schmidt, Ludwig and Tsipras, Dimitris and Vladu, Adrian},
	year = {2018},
}

@inproceedings{wong_fast_2019,
	title = {Fast is better than free: {Revisiting} adversarial training},
	booktitle = {ICLR},
	author = {Wong, Eric and Rice, Leslie and Kolter, J. Zico},
	year = {2019},
}

@inproceedings{liu_gandef_2019,
	address = {Cham},
	series = {{IFIP} {Advances} in {Information} and {Communication} {Technology}},
	title = {{GanDef}: {A} {GAN} {Based} {Adversarial} {Training} {Defense} for {Neural} {Network} {Classifier}},
	isbn = {978-3-030-22312-0},
	shorttitle = {{GanDef}},
	doi = {10.1007/978-3-030-22312-0_2},
	language = {en},
	booktitle = {{ICT} {Systems} {Security} and {Privacy} {Protection}},
	publisher = {Springer International Publishing},
	author = {Liu, Guanxiong and Khalil, Issa and Khreishah, Abdallah},
	editor = {Dhillon, Gurpreet and Karlsson, Fredrik and Hedström, Karin and Zúquete, André},
	year = {2019},
	pages = {19--32},
}

@inproceedings{kireev2022effectiveness,
  title={On the effectiveness of adversarial training against common corruptions},
  author={Kireev, Klim and Andriushchenko, Maksym and Flammarion, Nicolas},
  booktitle={UAI},
  pages={1012--1021},
  year={2022},
}

@inproceedings{shafahi_adversarial_2019,
	title = {Adversarial training for free!},
	booktitle = {NeurIPS},
	author = {Shafahi, Ali and Najibi, Mahyar and Ghiasi, Mohammad Amin and Xu, Zheng and Dickerson, John and Studer, Christoph and Davis, Larry S and Taylor, Gavin and Goldstein, Tom},
	year = {2019},
}

@article{zou2024improving,
  title={Improving Alignment and Robustness with Short Circuiting},
  author={Zou, Andy and Phan, Long and Wang, Justin and Duenas, Derek and Lin, Maxwell and Andriushchenko, Maksym and Wang, Rowan and Kolter, Zico and Fredrikson, Matt and Hendrycks, Dan},
  journal={arXiv preprint arXiv:2406.04313},
  year={2024}
}

@article{casper2024defending,
  title={Defending Against Unforeseen Failure Modes with Latent Adversarial Training},
  author={Casper, Stephen and Schulze, Lennart and Patel, Oam and Hadfield-Menell, Dylan},
  journal={arXiv preprint arXiv:2403.05030},
  year={2024}
}

@article{dekany2025mixat,
  title={MixAT: Combining Continuous and Discrete Adversarial Training for LLMs},
  author={D{\'e}k{\'a}ny, Csaba and Balauca, Stefan and Staab, Robin and Dimitrov, Dimitar I and Vechev, Martin},
  journal={arXiv preprint arXiv:2505.16947},
  year={2025}
}

@inproceedings{wang2023better,
  title={Better diffusion models further improve adversarial training},
  author={Wang, Zekai and Pang, Tianyu and Du, Chao and Lin, Min and Liu, Weiwei and Yan, Shuicheng},
  booktitle={ICML},
  year={2023},
}

@article{sheshadri2024latent,
  title={Latent adversarial training improves robustness to persistent harmful behaviors in llms},
  author={Sheshadri, Abhay and Ewart, Aidan and Guo, Phillip and Lynch, Aengus and Wu, Cindy and Hebbar, Vivek and Sleight, Henry and Stickland, Asa Cooper and Perez, Ethan and Hadfield-Menell, Dylan and others},
  journal={arXiv preprint arXiv:2407.15549},
  year={2024}
}

@inproceedings{
yu2025robust,
title={Robust {LLM} safeguarding via refusal feature adversarial training},
author={Lei Yu and Virginie Do and Karen Hambardzumyan and Nicola Cancedda},
booktitle={The Thirteenth International Conference on Learning Representations},
year={2025},
url={https://openreview.net/forum?id=s5orchdb33}
}

@article{liu2024adversarial,
  title={Adversarial tuning: Defending against jailbreak attacks for llms},
  author={Liu, Fan and Xu, Zhao and Liu, Hao},
  journal={arXiv preprint arXiv:2406.06622},
  year={2024}
}

@article{yong2023low,
  title={Low-resource languages jailbreak gpt-4},
  author={Yong, Zheng-Xin and Menghini, Cristina and Bach, Stephen H},
  journal={arXiv preprint arXiv:2310.02446},
  year={2023}
}

@article{ouyang2022training,
  title={Training language models to follow instructions with human feedback},
  author={Ouyang, Long and Wu, Jeffrey and Jiang, Xu and Almeida, Diogo and Wainwright, Carroll and Mishkin, Pamela and Zhang, Chong and Agarwal, Sandhini and Slama, Katarina and Ray, Alex and others},
  journal={NeurIPS},
  year={2022}
}

@article{mazeika2024harmbench,
  title={Harmbench: A standardized evaluation framework for automated red teaming and robust refusal},
  author={Mazeika, Mantas and Phan, Long and Yin, Xuwang and Zou, Andy and Wang, Zifan and Mu, Norman and Sakhaee, Elham and Li, Nathaniel and Basart, Steven and Li, Bo and others},
  journal={arXiv preprint arXiv:2402.04249},
  year={2024}
}

@article{andriushchenko2024jailbreaking,
  title={Jailbreaking leading safety-aligned llms with simple adaptive attacks},
  author={Andriushchenko, Maksym and Croce, Francesco and Flammarion, Nicolas},
  journal={arXiv preprint arXiv:2404.02151},
  year={2024}
}

@article{zou2023universal,
  title={Universal and transferable adversarial attacks on aligned language models},
  author={Zou, Andy and Wang, Zifan and Kolter, J Zico and Fredrikson, Matt},
  journal={arXiv preprint arXiv:2307.15043},
  year={2023}
}

@article{lapid2023open,
  title={Open Sesame! Universal Black Box Jailbreaking of Large Language Models},
  author={Lapid, Raz and Langberg, Ron and Sipper, Moshe},
  journal={arXiv preprint arXiv:2309.01446},
  year={2023}
}

@misc{zou2024circuitbreaker,
title={Improving Alignment and Robustness with Circuit Breakers},
author={Andy Zou and Long Phan and Justin Wang and Derek Duenas and Maxwell Lin and Maksym Andriushchenko and Rowan Wang and Zico Kolter and Matt Fredrikson and Dan Hendrycks},
year={2024},
eprint={2406.04313},
archivePrefix={arXiv},
primaryClass={cs.LG}
}

@article{zhu2023autodan,
  title={Autodan: Automatic and interpretable adversarial attacks on large language models},
  author={Zhu, Sicheng and Zhang, Ruiyi and An, Bang and Wu, Gang and Barrow, Joe and Wang, Zichao and Huang, Furong and Nenkova, Ani and Sun, Tong},
  journal={arXiv preprint arXiv:2310.15140},
  year={2023}
}

@article{li2024llm,
  title={Llm defenses are not robust to multi-turn human jailbreaks yet},
  author={Li, Nathaniel and Han, Ziwen and Steneker, Ian and Primack, Willow and Goodside, Riley and Zhang, Hugh and Wang, Zifan and Menghini, Cristina and Yue, Summer},
  journal={arXiv preprint arXiv:2408.15221},
  year={2024}
}

@article{chao2024jailbreakbench,
  title={Jailbreakbench: An open robustness benchmark for jailbreaking large language models},
  author={Chao, Patrick and Debenedetti, Edoardo and Robey, Alexander and Andriushchenko, Maksym and Croce, Francesco and Sehwag, Vikash and Dobriban, Edgar and Flammarion, Nicolas and Pappas, George J and Tramer, Florian and others},
  journal={arXiv preprint arXiv:2404.01318},
  year={2024}
}

@article{andriushchenko2024does,
  title={Does Refusal Training in LLMs Generalize to the Past Tense?},
  author={Andriushchenko, Maksym and Flammarion, Nicolas},
  journal={arXiv preprint arXiv:2407.11969},
  year={2024}
}

@article{souly2024strongreject,
  title={A strongreject for empty jailbreaks},
  author={Souly, Alexandra and Lu, Qingyuan and Bowen, Dillon and Trinh, Tu and Hsieh, Elvis and Pandey, Sana and Abbeel, Pieter and Svegliato, Justin and Emmons, Scott and Watkins, Olivia and others},
  journal={arXiv preprint arXiv:2402.10260},
  year={2024}
}

@article{geisler2025reinforce,
  title={{REINFORCE} Adversarial Attacks on Large Language Models: An Adaptive, Distributional, and Semantic Objective},
  author={Geisler, Simon and Wollschl{\"a}ger, Tom and Abdalla, MHI and Cohen-Addad, Vincent and Gasteiger, Johannes and G{\"u}nnemann, Stephan},
  journal={arXiv preprint arXiv:2502.17254},
  year={2025}
}

@article{hughes2024best,
  title={Best-of-n jailbreaking},
  author={Hughes, John and Price, Sara and Lynch, Aengus and Schaeffer, Rylan and Barez, Fazl and Koyejo, Sanmi and Sleight, Henry and Jones, Erik and Perez, Ethan and Sharma, Mrinank},
  journal={arXiv preprint arXiv:2412.03556},
  year={2024}
}

@article{geisler2024attacking,
  title={Attacking Large Language Models with Projected Gradient Descent},
  author={Geisler, Simon and Wollschl{\"a}ger, Tom and Abdalla, MHI and Gasteiger, Johannes and G{\"u}nnemann, Stephan},
  journal={arXiv preprint arXiv:2402.09154},
  year={2024}
}

@article{chao2023jailbreakingpair,
  title={Jailbreaking black box large language models in twenty queries},
  author={Chao, Patrick and Robey, Alexander and Dobriban, Edgar and Hassani, Hamed and Pappas, George J and Wong, Eric},
  journal={arXiv preprint arXiv:2310.08419},
  year={2023}
}

@misc{nie2025largelanguagediffusionmodels,
      title={Large Language Diffusion Models}, 
      author={Shen Nie and Fengqi Zhu and Zebin You and Xiaolu Zhang and Jingyang Ou and Jun Hu and Jun Zhou and Yankai Lin and Ji-Rong Wen and Chongxuan Li},
      year={2025},
      eprint={2502.09992},
      archivePrefix={arXiv},
      primaryClass={cs.CL},
      url={https://arxiv.org/abs/2502.09992}, 
}

@misc{zhu2025llada15variancereducedpreference,
      title={LLaDA 1.5: Variance-Reduced Preference Optimization for Large Language Diffusion Models}, 
      author={Fengqi Zhu and Rongzhen Wang and Shen Nie and Xiaolu Zhang and Chunwei Wu and Jun Hu and Jun Zhou and Jianfei Chen and Yankai Lin and Ji-Rong Wen and Chongxuan Li},
      year={2025},
      eprint={2505.19223},
      archivePrefix={arXiv},
      primaryClass={cs.LG},
      url={https://arxiv.org/abs/2505.19223}, 
}

@inproceedings{schmidt2018adversarially,
  title={Adversarially robust generalization requires more data},
  author={Schmidt, Ludwig and Santurkar, Shibani and Tsipras, Dimitris and Talwar, Kunal and Madry, Aleksander},
  booktitle={NeurIPS},
  year={2018}
}

@inproceedings{szegedy_intriguing_2014,
	title = {Intriguing properties of neural networks},
	booktitle = {ICLR},
	author = {Szegedy, Christian and Zaremba, Wojciech and Sutskever, Ilya and Bruna, Joan and Erhan, Dumitru and Goodfellow, Ian J. and Fergus, Rob},
	year = {2014},
}

@article{grattafiori2024llama,
  title={The llama 3 herd of models},
  author={Grattafiori, Aaron and Dubey, Abhimanyu and Jauhri, Abhinav and Pandey, Abhinav and Kadian, Abhishek and Al-Dahle, Ahmad and Letman, Aiesha and Mathur, Akhil and Schelten, Alan and Vaughan, Alex and others},
  journal={arXiv preprint arXiv:2407.21783},
  year={2024}
}

@misc{qwen2025qwen25technicalreport,
      title={Qwen2.5 Technical Report}, 
      author={Qwen and : and An Yang and Baosong Yang and Beichen Zhang and Binyuan Hui and Bo Zheng and Bowen Yu and Chengyuan Li and Dayiheng Liu and Fei Huang and Haoran Wei and Huan Lin and Jian Yang and Jianhong Tu and Jianwei Zhang and Jianxin Yang and Jiaxi Yang and Jingren Zhou and Junyang Lin and Kai Dang and Keming Lu and Keqin Bao and Kexin Yang and Le Yu and Mei Li and Mingfeng Xue and Pei Zhang and Qin Zhu and Rui Men and Runji Lin and Tianhao Li and Tianyi Tang and Tingyu Xia and Xingzhang Ren and Xuancheng Ren and Yang Fan and Yang Su and Yichang Zhang and Yu Wan and Yuqiong Liu and Zeyu Cui and Zhenru Zhang and Zihan Qiu},
      year={2025},
      eprint={2412.15115},
      archivePrefix={arXiv},
      primaryClass={cs.CL},
      url={https://arxiv.org/abs/2412.15115}, 
}

@article{tunstall2023zephyr,
  title={Zephyr: Direct distillation of lm alignment},
  author={Tunstall, Lewis and Beeching, Edward and Lambert, Nathan and Rajani, Nazneen and Rasul, Kashif and Belkada, Younes and Huang, Shengyi and Von Werra, Leandro and Fourrier, Cl{\'e}mentine and Habib, Nathan and others},
  journal={arXiv preprint arXiv:2310.16944},
  year={2023}
}

@misc{hendrycks2021measuringmassivemultitasklanguage,
      title={Measuring Massive Multitask Language Understanding}, 
      author={Dan Hendrycks and Collin Burns and Steven Basart and Andy Zou and Mantas Mazeika and Dawn Song and Jacob Steinhardt},
      year={2021},
      eprint={2009.03300},
      archivePrefix={arXiv},
      primaryClass={cs.CY},
      url={https://arxiv.org/abs/2009.03300}, 
}

@inproceedings{reimers-gurevych-2019-sentence,
    title = "Sentence-{BERT}: Sentence Embeddings using {S}iamese {BERT}-Networks",
    author = "Reimers, Nils  and
      Gurevych, Iryna",
    editor = "Inui, Kentaro  and
      Jiang, Jing  and
      Ng, Vincent  and
      Wan, Xiaojun",
    booktitle = "Proceedings of the 2019 Conference on Empirical Methods in Natural Language Processing and the 9th International Joint Conference on Natural Language Processing (EMNLP-IJCNLP)",
    month = nov,
    year = "2019",
    address = "Hong Kong, China",
    publisher = "Association for Computational Linguistics",
    url = "https://aclanthology.org/D19-1410/",
    doi = "10.18653/v1/D19-1410",
    pages = "3982--3992"
}

@article{team2025gemma,
  title={Gemma 3 technical report},
  author={{Gemma Team} and Kamath, Aishwarya and Ferret, Johan and Pathak, Shreya and Vieillard, Nino and Merhej, Ramona and Perrin, Sarah and Matejovicova, Tatiana and Ram{\'e}, Alexandre and Rivi{\`e}re, Morgane and others},
  journal={arXiv preprint arXiv:2503.19786},
  year={2025}
}

@article{allenai:arc,
      author    = {Peter Clark  and Isaac Cowhey and Oren Etzioni and Tushar Khot and
                    Ashish Sabharwal and Carissa Schoenick and Oyvind Tafjord},
      title     = {Think you have Solved Question Answering? Try ARC, the AI2 Reasoning Challenge},
      journal   = {arXiv:1803.05457v1},
      year      = {2018},
}

@inproceedings{rottger-etal-2024-xstest,
    title = "{XST}est: A Test Suite for Identifying Exaggerated Safety Behaviours in Large Language Models",
    author = {R{\"o}ttger, Paul  and
      Kirk, Hannah  and
      Vidgen, Bertie  and
      Attanasio, Giuseppe  and
      Bianchi, Federico  and
      Hovy, Dirk},
    editor = "Duh, Kevin  and
      Gomez, Helena  and
      Bethard, Steven",
    booktitle = "Proceedings of the 2024 Conference of the North American Chapter of the Association for Computational Linguistics: Human Language Technologies (Volume 1: Long Papers)",
    month = jun,
    year = "2024",
    address = "Mexico City, Mexico",
    publisher = "Association for Computational Linguistics",
    url = "https://aclanthology.org/2024.naacl-long.301/",
    doi = "10.18653/v1/2024.naacl-long.301",
    pages = "5377--5400"
}

@article{beyer2025llm,
  title={Llm-safety evaluations lack robustness},
  author={Beyer, Tim and Xhonneux, Sophie and Geisler, Simon and Gidel, Gauthier and Schwinn, Leo and G{\"u}nnemann, Stephan},
  journal={arXiv preprint arXiv:2503.02574},
  year={2025}
}

@article{beyer2025sampling,
  title={Sampling-aware adversarial attacks against large language models},
  author={Beyer, Tim and Scholten, Yan and Schwinn, Leo and G{\"u}nnemann, Stephan},
  journal={arXiv preprint arXiv:2507.04446},
  year={2025}
}

@inproceedings{schwinn2023adversarial,
  title={Adversarial attacks and defenses in large language models: Old and new threats},
  author={Schwinn, Leo and Dobre, David and G{\"u}nnemann, Stephan and Gidel, Gauthier},
  booktitle={NeurIPS, ICBINB Workshop},
  year={2023},
}

@inproceedings{scholten2025a,
title={A Probabilistic Perspective on Unlearning and Alignment for Large Language Models},
author={Yan Scholten and Stephan G{\"u}nnemann and Leo Schwinn},
booktitle={ICLR},
year={2025},
}

@inproceedings{schwinn2024soft,
title={Soft Prompt Threats: Attacking Safety Alignment and Unlearning in Open-Source {LLM}s through the Embedding Space},
author={Leo Schwinn and David Dobre and Sophie Xhonneux and Gauthier Gidel and Stephan G{\"u}nnemann},
booktitle={NeurIPS},
year={2024},
}

@article{xhonneux2024efficient,
  title={Efficient adversarial training in llms with continuous attacks},
  author={Xhonneux, Sophie and Sordoni, Alessandro and G{\"u}nnemann, Stephan and Gidel, Gauthier and Schwinn, Leo},
  journal={NeurIPS},
  year={2024}
}

@article{ludke2025diffusion,
  title={Diffusion LLMs are Natural Adversaries for any LLM},
  author={L{\"u}dke, David and Wollschl{\"a}ger, Tom and Ungermann, Paul and G{\"u}nnemann, Stephan and Schwinn, Leo},
  journal={arXiv preprint arXiv:2511.00203},
  year={2025}
}

@article{beyer2025fast,
  title={Fast Proxies for LLM Robustness Evaluation},
  author={Beyer, Tim and Schuchardt, Jan and Schwinn, Leo and G{\"u}nnemann, Stephan},
  journal={arXiv preprint arXiv:2502.10487},
  year={2025}
}

@article{beyer2025adversariallm,
  title={AdversariaLLM: A Unified and Modular Toolbox for LLM Robustness Research},
  author={Beyer, Tim and Dornbusch, Jonas and Steimle, Jakob and Ladenburger, Moritz and Schwinn, Leo and G{\"u}nnemann, Stephan},
  journal={arXiv preprint arXiv:2511.04316},
  year={2025}
}

@article{schwinn2025adversarial,
  title={Adversarial alignment for llms requires simpler, reproducible, and more measurable objectives},
  author={Schwinn, Leo and Scholten, Yan and Wollschl{\"a}ger, Tom and Xhonneux, Sophie and Casper, Stephen and G{\"u}nnemann, Stephan and Gidel, Gauthier},
  journal={arXiv preprint arXiv:2502.11910},
  year={2025}
}

@article{gibbs2002choosingboundingprobabilitymetrics,
      title={On choosing and bounding probability metrics}, 
      author={Alison L. Gibbs and Francis Edward Su},
      year={2002},
      journal={arXiv preprint arXiv:math/0209021},
      eprint={math/0209021},
      archivePrefix={arXiv},
      primaryClass={math.PR},
      url={https://arxiv.org/abs/math/0209021}, 
}
\bibliographystyle{icml2026}

\clearpage
\appendix
\onecolumn
\section{Proof of Theorem 1}
\label{app:proof_theorem1}

\subsection{Setup}
Let $q$ denote a natural language distribution over $(x,y) \in \mathcal{Z}$ for a set of token sequences $\mathcal{Z}$.
Further, let $h:\mathcal{Z}\to\{0,1\}$ be a harmfulness indicator, where $h(y)=1$ denotes a harmful response.
We define the distribution of prompts and their harmful responses as
\[
\tilde q(x,y)
\;:=\;
q\bigl(x,y \mid h(y)=1\bigr)
\;=\;
\frac{q(x,y)\,\mathbf 1\{h(y)=1\}}{q(h(y)=1)}.
\]
We let $p^{\mathrm{AR}}_\theta(y|x)$ denote the target autoregressive model.
We define the robust loss
\[
\ell_{rob}(x, y; \theta)
\;:=\;
\sup_{\delta \in \Delta}\ \mathcal{L}\!\bigl(p^{\mathrm{AR}}_\theta(y\mid x+\delta)\bigr),
\]
where $\Delta$ is a perturbation set (e.g.\ in embedding space) and use the standard total variation distance
\[
\TV(P,Q)
\;:=\;
\sup_{A}\bigl|P(A)-Q(A)\bigr|
\;=\;
\frac{1}{2}\,\|P-Q\|_1.
\]

\paragraph{Assumptions}
We assume this loss is bounded, i.e.,
\begin{equation}
    |\ell_{rob}(x,y;\theta)| \le M \qquad \text{for all } (x,y)\in\mathcal{Z}.
    \label{eq:rob-loss-bounded}
\end{equation}
This assumption can be enforced in practice by \emph{loss clipping} (e.g.\ replacing $\ell_{rob}$ with
$\mathrm{clip}(\ell_{rob},-M,M)$), and it is used below to apply a TV-based bound.

Further, we assume that the diffusion surrogate satisfies the conditional fidelity bound
\[
\mathbb{E}_{y\sim \tilde q(y)}
\left[
\TV\!\bigl(q(\cdot\mid y),p^{\mathrm{diff}}_\theta(\cdot\mid y)\bigr)
\right]\le \varepsilon.
\]

\paragraph{True population robust risk.}
\begin{equation}
    \mathcal{R}_{pop}(\theta) = \mathbb{E}_{(x,y) \sim \tilde{q}} [\ell_{rob}(x, y; \theta)].
\end{equation}

\paragraph{Surrogate (DAT) risk.}
Our surrogate objective replaces the conditional distribution $\tilde{q}(x|y)$ with the diffusion surrogate
$p_{\theta}^{\mathrm{diff}}(x|y)$:
\begin{equation}
    \mathcal{R}_{diff}(\theta) = \mathbb{E}_{y \sim \tilde{q}(y)} \mathbb{E}_{x \sim p_{\theta}^{\mathrm{diff}}(x|y)} [\ell_{rob}(x, y; \theta)].
\end{equation}

\paragraph{Restriction to harmful $y$.}
Since $\tilde q$ is the restriction of $q$ to harmful responses, its marginal satisfies
\[
\tilde q(y) \;=\; q\bigl(y \mid h(y)=1\bigr).
\]
Moreover, for any $y$ with $h(y)=1$ (equivalently, whenever $\tilde q(y)>0$), the conditional over prompts is unchanged:
\begin{equation}
    \tilde q(x\mid y)
    \;=\;
    \frac{\tilde q(x,y)}{\tilde q(y)}
    \;=\;
    \frac{q(x,y)}{q(y)}
    \;=\;
    q(x\mid y).
    \label{eq:conditional-unchanged}
\end{equation}

\subsection{Proof}
We wish to bound the absolute difference $|\mathcal{R}_{pop}(\theta) - \mathcal{R}_{diff}(\theta)|$.

\begin{proof}

By the law of iterated expectations under $\tilde q$ (equivalently, rearranging sums since $\mathcal{Z}$ is discrete) and using~\eqref{eq:conditional-unchanged}, for harmful $y$ we may equivalently view the inner conditional $x \sim \tilde{q}(x \mid y)$ as $x \sim q(x\mid y)$,
\begin{equation}\label{eq:Rpop-iterated}
\begin{aligned}
\mathcal{R}_{pop}(\theta)
&=
\mathbb{E}_{(x,y)\sim \tilde q}\bigl[\ell_{rob}(x,y;\theta)\bigr] \\
&=
\mathbb{E}_{y\sim \tilde q(y)}\ \mathbb{E}_{x\sim \tilde q(x\mid y)}\bigl[\ell_{rob}(x,y;\theta)\bigr] \\
&=
\mathbb{E}_{y\sim \tilde q(y)}\ \mathbb{E}_{x\sim q(x\mid y)}\bigl[\ell_{rob}(x,y;\theta)\bigr].
\end{aligned}
\end{equation}

Combining~\eqref{eq:Rpop-iterated} with the definition of $\mathcal{R}_{diff}$ gives
\begin{align}
    |\mathcal{R}_{pop}(\theta) - \mathcal{R}_{diff}(\theta)|
    &=
    \left|
        \mathbb{E}_{y \sim \tilde{q}(y)}
        \left[
            \mathbb{E}_{x \sim q(x\mid y)} [\ell_{rob}]
            -
            \mathbb{E}_{x \sim p_{\theta}^{\mathrm{diff}}(x\mid y)} [\ell_{rob}]
        \right]
    \right|
    \label{eq:diff-1}
    \\
    &\le
    \mathbb{E}_{y \sim \tilde{q}(y)}
    \left[
        \left|
            \mathbb{E}_{x \sim q(x\mid y)} [\ell_{rob}]
            -
            \mathbb{E}_{x \sim p_{\theta}^{\mathrm{diff}}(x\mid y)} [\ell_{rob}]
        \right|
    \right],
    \label{eq:diff-2}
\end{align}
where the triangle inequality is applied.

We use the following standard inequality for any function $f:\mathcal{Z}\to[-M,M]$ and
distributions $P,Q$ on $\mathcal{Z}$,
\begin{equation}
    \left| \mathbb{E}_{x \sim P} [f(x)] - \mathbb{E}_{x \sim Q} [f(x)] \right|
    \le
    2M \cdot \TV(P, Q),
    \label{eq:tv-ineq}
\end{equation}
where $\TV$ is total variation \cite{gibbs2002choosingboundingprobabilitymetrics}.
Applying~\eqref{eq:tv-ineq} to~\eqref{eq:diff-2} with
\[
P= q(\cdot\mid y),
\qquad
Q=p_\theta^{\mathrm{diff}}(\cdot\mid y),
\qquad
f(\cdot)=\ell_{rob}(\cdot,y;\theta),
\]
and using the boundedness assumption~\eqref{eq:rob-loss-bounded}, we obtain
\begin{equation}
    \left|
        \mathbb{E}_{x \sim q(x\mid y)} [\ell_{rob}]
        -
        \mathbb{E}_{x \sim p_{\theta}^{\mathrm{diff}}(x\mid y)} [\ell_{rob}]
    \right|
    \le
    2M \cdot \TV\!\bigl(q(\cdot\mid y),\ p_{\theta}^{\mathrm{diff}}(\cdot\mid y)\bigr).
    \label{eq:inner-tv}
\end{equation}
Substituting~\eqref{eq:inner-tv} into~\eqref{eq:diff-2} yields
\begin{align}
    |\mathcal{R}_{pop}(\theta) - \mathcal{R}_{diff}(\theta)|
    &\le
    \mathbb{E}_{y \sim \tilde{q}(y)}
    \left[
        2M \cdot \TV\!\bigl(q(\cdot\mid y),\ p_{\theta}^{\mathrm{diff}}(\cdot\mid y)\bigr)
    \right]
    \nonumber\\
    &=
    2M\cdot
    \mathbb{E}_{y \sim \tilde{q}(y)}
    \left[
        \TV\!\bigl(q(\cdot\mid y),\ p_{\theta}^{\mathrm{diff}}(\cdot\mid y)\bigr)
    \right].
    \label{eq:outer-tv}
\end{align}

Finally, using the fidelity assumption we obtain
\[
|\mathcal{R}_{pop}(\theta) - \mathcal{R}_{diff}(\theta)|
\le
2M\cdot \varepsilon.
\]
\end{proof}

\noindent This concludes the proof. The bound shows that the approximation error of our surrogate objective is linear in the
\emph{fidelity} of the diffusion surrogate, measured as an expected conditional TV distance under the harmful-response
marginal $y\sim\tilde q(y)$.

\section{Experiment Details}
\label{app:exp_details}


    \subsection{Models and Adapters}\label{app:Models}

\begin{table}[H]
    \centering
    \small
    \caption{Sources of Hugging Face models and adapters.}
    \label{tab:hf_models_adapters}
    \begin{tabular}{l l l}
    \toprule
    \textbf{Base Model} & \textbf{Adapter} & \textbf{HF Source} \\
    \midrule
    \multirow{4}{*}{Llama3-8B}
        & --        & \texttt{meta-llama/Meta-Llama-3-8B-Instruct} \\
        & CB        & \texttt{GraySwanAI/Llama-3-8B-Instruct-RR} \\
        & LAT       & \texttt{LLM-LAT/robust-llama3-8b-instruct} \\
        & MixAT-GCG & \texttt{INSAIT-Institute/Llama3-8B-MixAT-GCG} \\
    \midrule
    \multirow{2}{*}{Qwen2.5-14B}
        & --        & \texttt{Qwen/Qwen2.5-14B-Instruct} \\
        & MixAT-GCG & \texttt{INSAIT-Institute/Qwen-14B-MixAT-GCG} \\
    \midrule
    \multirow{1}{*}{Qwen2.5-7B}
        & --        & \texttt{Qwen/Qwen2.5-7B-Instruct} \\
    \multirow{1}{*}{Zephyr-7B}
        & --        & \texttt{HuggingFaceH4/zephyr-7b-beta} \\
    \multirow{1}{*}{Gemma3-12B}
        & --        & \texttt{google/gemma-3-12b-it} \\
    \multirow{1}{*}{LLaDA-8B}
        & --        & \texttt{GSAI-ML/LLaDA-8B-Base} \\
    \bottomrule
    \end{tabular}
\end{table}

\subsection{Training}\label{app:Training}

\begin{table}[H]
\centering
\small
\caption{Hyperparameters for models trained with \DAT{}.}
\label{tab:dat_hparams}
\begin{tabular}{lcc}
\toprule
\textbf{Hyperparameter} & \textbf{Llama3-8B} & \textbf{Qwen2.5-14B} \\
\midrule
Learning rate           & $1\times10^{-4}$ & $1\times10^{-4}$ \\
Batch size              & 8                & 8                \\
Number of epochs        & 6                & 6                \\
Optimizer               & AdamW            & AdamW            \\
Adv.\ learning rate     & $1\times10^{-1}$ & $1\times10^{-1}$ \\
Attack-steps            & 20              &  40            \\
Attack-radius           & $1\times10^{7}$  & $1\times10^{7}$  \\
Utility data ratio      & 0.5              & 0.5              \\
Max seq.\ length        & 128              & 128              \\
Away weight             & 0.5              & 0.3              \\
Toward weight           & 0.5              & 0.7              \\
Utility weight          & 0.45             & 0.25             \\
Quantization            & none (bf16)      & none (bf16)      \\
\bottomrule
\end{tabular}
\end{table}

\begin{table}[H]
\centering
\small
\caption{Hyperparameters for models trained with CAT.}
\label{tab:cat_hparams}
\begin{tabular}{lcc}
\toprule
\textbf{Hyperparameter} & \textbf{Llama3-8B} & \textbf{Qwen2.5-14B} \\
\midrule
Learning rate           & $2\times10^{-4}$ & $1\times10^{-4}$ \\
Batch size              & 16                & 16                \\
Number of epochs        & 6                & 6                \\
Optimizer               & AdamW            & AdamW            \\
Attack-radius           & $1\times10^{7}$  & $1\times10^{7}$  \\
Adv.\ learning rate     & $1\times10^{-1}$ & $1\times10^{-1}$ \\
Attack-steps            & 20              &  20            \\
Utility data ratio      & 0.5              & 0.5              \\
Max seq.\ length        & 128              & 128              \\
Away weight             & 0.5              & 0.3              \\
Toward weight           & 0.5              & 0.7              \\
Utility weight          & 0.25             & 0.25             \\
Quantization            & none (fp16)      & none (bf16)      \\
\bottomrule
\end{tabular}
\end{table}

\subsection{Attacks}\label{app:Attacks}
We used the attack code base \href{https://github.com/LLM-QC/AdversariaLLM}{AdversariaLLM} \cite{beyer2025adversariallm}. If an attack uses multiple steps or produces different versions of attacks, we prompt the target model on each candidate and report success if at least one candidate achieves a successful judge score. Unless stated otherwise, we use greedy decoding when sampling the target mode.
\begin{itemize}
    \item \textbf{GCG:} We run a suffix attack with the suffix initialized with "x x x x x x x x x x x x x x x x x x x x" for $250$ steps, search width $512$ and select the top-$256$ most promising candidates.
    \item \textbf{PAIR:} We run for $20$ steps with one stream. We use Gemma 3 12B Instruct \cite{team2025gemma} as the chosen attacker model.
    \item \textbf{Direct sampling:} We sample $1000$ generations for each unperturbed prompt using multinomial sampling with a temperature of $1.0$.
    \item \textbf{Inpainting:} We use a dataset of $1024$ diffusion attack prompts.
    \item \textbf{Best-of-N:} We generate $1000$ perturbed versions of each prompt and sample a single generation for each with a temperature of $1.0$. We apply the default perturbation strength $\sigma = 0.4$, and allow all perturbations (word scrambling, capitalization, ascii perturbations).
    \item \textbf{Best-of-All:} For a particular original instruction and target, we determine the ensemble to be successful if any of the attacks was successful.   
\end{itemize}

\end{document}